\documentclass{tlp}
\usepackage{mathptmx}

\usepackage{tikz}
\usepackage{calc}
\usetikzlibrary{calc}
\usepackage{svg}
\usepackage{xspace}
\usepackage{color}

\usepackage{helvet}  
\usepackage{courier}  
\usepackage{caption} 

\usepackage[utf8]{inputenc}
\usepackage{booktabs}
\usepackage{amssymb}
\usepackage[switch]{lineno}
\usepackage{amsmath}

\usepackage{subcaption}

\usepackage{listings}
\usepackage{tikz}
\usepackage{multirow}
\usepackage{lscape}
\usepackage{longtable}

\usepackage{mathtools}
\usepackage{cancel}
\usepackage{placeins} 

\usepackage{xcolor}
\usepackage{amsmath}    
\usepackage{amssymb}    
\usepackage{mathtools}  
\usepackage{microtype}  
\usepackage[inline]{enumitem} 
\usepackage{multirow}   
\usepackage{booktabs}   
\usepackage{subcaption} 
\usepackage[ruled,linesnumbered,algosection,noend]{algorithm2e} 
\usepackage{url}

\newcommand{\longversion}[1]{}

\DeclareMathOperator{\dom}{dom}

\DeclareMathOperator{\var}{var}

\newcommand{\nag}{\texttt{NaGG}\xspace}
\newcommand{\ngrThree}{\texttt{newground3}\xspace}
\newcommand{\ngr}{\texttt{newground}\xspace}

\newcommand{\clingo}{\texttt{clingo}\xspace}
\newcommand{\gringo}{\texttt{gringo}\xspace}
\newcommand{\idlv}{\texttt{idlv}\xspace}
\newcommand{\dlv}{\texttt{dlv}\xspace}
\newcommand{\lpopt}{\texttt{Lpopt}\xspace}
\newcommand{\clasp}{\texttt{clasp}\xspace}
\newcommand{\wasp}{\texttt{wasp}\xspace}
\newcommand{\SOTAGrounding}{\protect\ensuremath{\mathcal{G}}\xspace}

\newcommand{\depGraph}{\ensuremath{\mathcal{D}}\xspace}

\newcommand{\varGraph}[1]{\ensuremath{\mathcal{D}\left(#1\right)}\xspace}

\allowdisplaybreaks
\newcommand{\lvlmap}{\ensuremath{\psi}}

\newcommand{\gprog}{P\xspace}
\newcommand{\prog}{\ensuremath{\Pi}\xspace}

\newcommand{\vecv}[0]{\mathbf}

\newcommand{\HybridGrounding}{\protect\ensuremath{\mathcal{H}}\xspace}
\newcommand{\AlphaG}{\protect{\texttt{ALPHA}}\xspace}
\newcommand{\ProASP}{\protect{\texttt{ProASP}}\xspace}
\newcommand{\MarkerSet}{\protect{\texttt{MARKER}}\xspace}
\newcommand{\NgrThreeG}{\protect{\texttt{NG-G}}\xspace}

\newcommand{\NgrThreeI}{\protect{\texttt{NG-I}}\xspace}

\newcommand{\ruleBody}{\ensuremath{B_r}\xspace}
\newcommand{\rulePosBody}{\ensuremath{B_r^+}\xspace}
\newcommand{\ruleNegBody}{\ensuremath{B_r^{-}}\xspace}
\newcommand{\ruleHead}{\ensuremath{H_r}\xspace}

\newcommand{\Card}[1]{\left|#1\right|}

\DeclareMathOperator{\scc}{SCC}

\newcommand{\eqdef}{\ensuremath{\,\mathrel{\mathop:}=}}

\newcounter{myenumctr}

\usetikzlibrary{matrix}

\tikzset{ 
	table/.style={
		row sep=-\pgflinewidth,
		column sep=-\pgflinewidth,
		nodes={rectangle,
			align=center,
				anchor=center},
	},
	row 1/.style={nodes={
		}},
}

\DeclareMathAlphabet\mathbfcal{OMS}{cmsy}{b}{n}

\pgfdeclarelayer{bg}    
\pgfsetlayers{bg,main}

\newtheorem{example}{Example}
\newtheorem{theorem}{Theorem}

\usepackage{comment}

{%
  \addtocounter{observation}{-1}
  \endgroup
}%

{%
  \addtocounter{corollary}{-1}
  \endgroup
}%

{%
  \addtocounter{proposition}{-1}
  \endgroup
}%

{%
  \addtocounter{lemma}{-1}
  \endgroup
}%

{%
  \addtocounter{thm}{-1}
  \endgroup
}%

\definecolor{backcolor}{rgb}{0.95,0.95,0.95}

\lstdefinestyle{mystyle}{
    backgroundcolor=\color{backcolor},
    breaklines=true,
    numbers=left,
    numbersep=5pt,
    basicstyle=\ttfamily\footnotesize,
    captionpos=b,
    literate={:-}{{$\leftarrow$}}1
             {!=}{{$\neq$}}1,
    escapeinside={<*>}{</*>},
    mathescape=true
}

\newcommand\bcmdtab{\noindent\bgroup\tabcolsep=0pt%
  \begin{tabular}{@{}p{10pc}@{}p{20pc}@{}}}
\newcommand\ecmdtab{\end{tabular}\egroup}

\begin{document}

\lefttitle{A. Beiser, M. Hecher, and S. Woltran}
\jnlPage{1}{8}
\jnlDoiYr{2021}
\doival{10.1017/xxxxx}

  \title[Automated Hybrid Grounding]
        {Automated Hybrid Grounding Using Structural and Data-Driven Heuristics}
\begin{authgrp}
        \author{\sn{Alexander} \gn{Beiser}}
        \author{\sn{Stefan} \gn{Woltran}}
        \affiliation{TU Wien, Favoritenstrasse 9–11, Vienna, 1040, Austria}
        \author{\sn{Markus} \gn{Hecher}}
        \affiliation{Univ. Artois, CNRS, UMR8188, Computer Science Research Center of Lens (CRIL), Lens, France}
\end{authgrp}

\pagerange{\pageref{firstpage}--\pageref{lastpage}}

\history{\sub{16 April 2025} \rev{21 June 2025} \acc{8 Jul 2025}}



\maketitle
  \begin{abstract}
  The grounding bottleneck poses one of the key challenges that hinders the widespread adoption of Answer Set Programming in industry.
  Hybrid Grounding is a step in alleviating the bottleneck by combining the strength of standard bottom-up grounding with recently proposed techniques where rule bodies are decoupled during grounding.
  However, it has remained unclear when hybrid grounding shall use body-decoupled grounding and when to use standard bottom-up grounding.
  In this paper, we address this issue by developing automated hybrid grounding: we introduce a splitting algorithm based on data-structural
  heuristics that detects when to use  body-decoupled grounding and when 
  standard grounding is beneficial.
  We base our heuristics on the structure of rules and an estimation procedure that incorporates the data of the instance.
  The experiments conducted on our prototypical implementation demonstrate promising results,
  which show an improvement on hard-to-ground scenarios, whereas on hard-to-solve instances, we approach state-of-the-art performance\footnote{Supplementary material and prototype available under: \url{https://github.com/alexl4123/newground}}.  
  \end{abstract}

%
%
\section{Introduction}
\label{sec:introduction}

The so-called \emph{grounding bottleneck}~\cite{gebser_evaluation_2018,tsamoura_beyond_2020} in
Answer Set Programming (ASP) 
is one of the key factors that hinders 
large scale
adoption of ASP in the industry~\cite{falkner_industrial_2018}.
It occurs as part of the grounding step~\cite{kaminski_foundations_2023}, which is an integral part of the state-of-the-art
answer set programming systems, such as \clingo~\cite{gebser_theory_2016} or \dlv~\cite{leone_dlv_2006}.
Grounding replaces the variables of a non-ground ASP program by their domain values, which inherently results in an exponentially larger~\cite{dantsin_complexity_2001} ground program.

The grounding bottleneck is a long-standing problem, which is the reason why modern grounders like \gringo~\cite{gebser_abstract_2015} or \idlv~\cite{calimeri_i-dlv_2017}, are highly optimized systems.
They work according to a bottom-up and semi-naive  approach~\cite{gebser_grounding_2015}, which instantiates rules along their occurrence on the topological order of the dependency graph of the program.
Although these systems are highly optimized and implement advanced rewriting methods, as they incorporate structural information on rules~\cite{calimeri_optimizing_2018,bichler_lpopt_2016},
they are exponential in the number of variables in the worst case.

Body-decoupled grounding (BDG)~\cite{besin_body-decoupled_2022} is a novel approach that alleviates the grounding bottleneck by decomposing rules into literals and grounding the literals individually.
This is achieved by shifting some of the grounding effort from the grounder to the solver, thereby exploiting the power of modern ASP solving technology.
Practically, BDG's grounding size is only dependent on the maximum arity $a$ of a program.
Experiments on grounding-heavy tasks have shown promising results,
by solving previously ungroundable instances.
However, BDG on its own is not interoperable with other state-of-the-art techniques, which prohibits BDG from playing to its strengths in practical settings.
Hybrid grounding~\cite{beiser_bypassing_2024} partially alleviates the challenge
of interoperability, by enabling the free (manual) partitioning of a program $\prog$ into a part $\prog_{\HybridGrounding}$ grounded by BDG and $\prog_{\SOTAGrounding}$ grounded by bottom-up grounding.

Still, it remains unclear when the usage of BDG is beneficial.
Grounding with BDG potentially increases
the solving time, 
as BDG pushes effort spent in grounding to solving.
Rewriting techniques, used for example in \idlv, complicate this matter further.
Additionally, BDG's grounding size is solely dependent on the domain,
not considering the peculiarities of the instance.
%

\noindent We address this challenge by introducing \textit{automated hybrid grounding}, which is an algorithm for detecting those parts of a program that shall be grounded by BDG.
Our contributions are~three-fold:
\begin{itemize}
    \item We present the data-structural splitting heuristics, which decides (based on the structure of a rule and the instance's data) whether it is beneficial to ground with BDG.
    \item We develop the prototype \ngrThree that integrates BDG into bottom-up  procedures of state-of-the-art grounders and uses BDG according to data-structural heuristics.
    \item Our experiments show that with \ngrThree we approach state-of-the-art performance on solving-heavy scenarios, while beating the state-of-the-art on grounding-heavy scenarios.
\end{itemize}

The paper is structured as follows.
After this introduction (Section~\ref{sec:introduction}), we state the necessary preliminaries of ASP and on grounding techniques (Section~\ref{sec:preliminaries}).
We continue by showing our data-structural heuristics (Section~\ref{chpt:automated-splitting:sec:automated-splitting-heuristics}).
Next is the high-level description of our prototypical implementation \ngrThree (Section~\ref{sec:novel-prototype}),
which is followed by the conducted experiments (Section~\ref{sec:automatic-splitting:experiments}).
The paper ends with a conclusion and discussion (Section~\ref{sec:discussion_conclusion}).
\smallskip

\noindent
\textbf{Related work}. While state-of-the-art grounders use semi-naive grounding techniques~\cite{gebser_grounding_2015,calimeri_i-dlv_2017},
we focus on the interoperability between state-of-the-art grounders and alternative grounding procedures.
Alternative grounding procedures include lazy-grounding~\cite{weinzierl_blending_2017,weinzierl_advancing_2020}, 
lazy-grounding with heuristics~\cite{leutgeb_techniques_2018},
compilation-based techniques via lazy rule injection~\cite{cuteri_partial_2019,lierler_dualgrounder_2021}, or
compilation-based techniques via extensions of the CDNL procedure~\cite{mazzotta_compilation_2022,dodaro_compilation_2023,dodaro_blending_2024}. 
Approaches based on ASP Modulo Theory combine ASP with methods from other 
fields~\cite{banbara_clingcon_2017,balduccini_constraint_2017,brewka_answer_2012}.
Structure-based techniques also showed promising results~\cite{bichler_lpopt_2016}.
We focus on the alternative grounding procedure of body-decoupled grounding (BDG)~\cite{besin_body-decoupled_2022}.
In contrast to the other approaches, BDG is a rewriting approach based on complexity theory.
%
In~\citeN{beiser_bypassing_2024} BDG was extended by hybrid grounding and the handling of aggregates.
Hybrid Grounding enables the free partitioning of a program into a part grounded by semi-naive grounding and a part grounded by BDG.
Aggregates are handled by specially crafted rewriting procedures that decouple aggregates.
We extend the previous work on BDG by proposing a splitting heuristics
that decides when the usage of BDG is useful.
Further, we provide an extensive empirical evaluation of the heuristics with our prototype \ngrThree.
Previously proposed splitting heuristics include heuristics on when to use
bottom-up grounding and when to use structural rewritings~\cite{calimeri_optimizing_2018}.
Related work proposes a machine learning-based heuristics~\cite{mastria_machine_2020}.
In contrast, we focus on a splitting heuristics, when the usage of BDG is beneficial.

\vspace{-0.7cm}
\section{Preliminaries}
\label{sec:preliminaries}

\textbf{Ground ASP}. A ground program \gprog consists of ground rules of the form 
$a_1 \lor \ldots \lor a_l \leftarrow$ $ a_{l+1}, \ldots, a_m,$ $ \neg a_{m+1}, \ldots, \neg a_{n}$,
where $a_i$ are propositional atoms and $l,m,n$ are non-negative integers with $l \leq m \leq n$.
We let $\ruleHead \coloneq \{a_1, \ldots, a_l\}$,
$\rulePosBody \coloneq \{a_{l+1}, \ldots, a_m\}$,
$\ruleNegBody \coloneq \{a_{m+1}, \ldots, a_n\}$,
and $\ruleBody \coloneq \rulePosBody \cup \ruleNegBody$.
$r \in \gprog$ is normal iff $\Card{H_r} \leq 1$, a constraint iff $\Card{H_r} = 0$, and disjunctive iff $\Card{H_r} > 1$.
The \emph{dependency graph} \depGraph is the directed graph $\depGraph = (V,E)$,
where $V = \bigcup_{r \in \gprog} \ruleHead \cup \ruleBody$
and $E = \{(b,h)_{+} | r \in \gprog, b \in \rulePosBody, h \in \ruleHead\} \cup \{(b,h)_{-} | r \in \gprog, b \in \ruleNegBody, h \in \ruleHead\}$.
We refer by $(b,h)_{+}$ to a positively labeled edge and by $(b,h)_{-}$ to a negatively labeled edge.
A positive cycle consists solely of positive edges.
A program \gprog is \emph{tight} iff there is no positive cycle in \depGraph,
\gprog is not stratified iff there is a cycle in \depGraph that contains at least one negative edge, and
\gprog is \emph{head-cycle-free (HCF)} iff there is no positive cycle in \depGraph among any two atoms $\{a,b\} \subseteq H_r$.
%
\textit{IsConstraint(r)} is true iff $r$ is a constraint.

We proceed by defining the semantics of ASP.
Let $\text{HB}(\gprog)$ be the Herbrand Base (the set of all atoms).
For ground programs this is $\text{HB}(\gprog) = \{p \mid r \in \gprog, p \in \ruleHead \cup \ruleBody\}$.
An \emph{interpretation} $I$ is a set of atoms $I \subseteq \text{HB}(\gprog)$.
$I$ \emph{satisfies} a rule~$r$ iff
$(H_r{\,\cup\,} B^-_r) {\,\cap\,} I {\,\neq\,} \emptyset$ or
$B^+_r {\,\setminus\,} I {\,\neq\,} \emptyset$. 
$I$ is a \emph{model} of $\gprog$
iff it satisfies all rules of~$\gprog$. 
A rule $r\in\gprog$ is \emph{suitable for justifying} $a \in I$ iff $a\in H_r$, $\rulePosBody\subseteq I$, and
$I \cap \ruleNegBody = I \cap (\ruleHead \setminus \{a\}) = \emptyset$.
A \emph{level mapping} $\lvlmap : I \rightarrow \{0, \ldots, |I|-1\}$ assigns every atom in~$I$ a unique value~\cite{lin_tight_2003,janhunen_translatability_2006}.
An atom~$a\in I$ is \emph{founded} iff
there is a rule $r\in\gprog$
s.t. (i) $r$ is suitable for justifying $a$
and (ii) there are no cyclic-derivations, i.e., 
$\forall b \in \rulePosBody: \lvlmap(b) < \lvlmap(a)$.
$I$ is an \emph{answer set} of a normal (HCF) program~$\gprog$ iff~$I$ is a model (satisfied) of~$\gprog$,
and all atoms in~$I$ are founded.
The \emph{Gelfond-Lifschitz (GL) reduct} is the classical way to define semantics.
The GL reduct of~$\gprog$ under~$I$ is the program~$\gprog^I$ obtained
from $\gprog$ by first removing all rules~$r$ with
$B^-_r{\,\cap\,} I\neq \emptyset$ and then removing all~$p \in \ruleNegBody$ from the remaining rules~$r$~\cite{gelfond_classical_1991}. 
%
$I$ is an \emph{answer set} of a program~$\gprog$ if $I$ is a \emph{
minimal model} (w.r.t.~$\subseteq$) of~$\gprog^I$. 

\textbf{Non-ground ASP.}
A non-ground program \prog consists of non-ground rules $r$ of the form 
$p_1(\vecv {X}_1) \vee \ldots \vee p_\ell(\vecv {X_\ell}) \leftarrow p_{\ell{+}1}(\vecv {X}_{\ell{+}1}), \ldots, p_{m}(\vecv {X}_{m}), \neg p_{m{+}1}(\vecv {X}_{m{+}1}), \ldots, \neg p_n({\vecv {X}_n})$,
where each $p_i(\vecv{X}_i)$ is a literal and
$l,m,n$ are non-negative integers s.t. $l \leq m \leq n$.
A literal $p_i(\vecv{X}_i)$ consists of a \emph{predicate} $p_i$
and a \emph{term} vector $\vecv{X}_i = \langle x_1, \ldots, x_z \rangle$.
A \emph{term} $x_j \in \vecv{X}_i$ is a constant or a variable.
For a predicate $p_i$ let $|\mathbf{X}_i|$ be its arity $a(p_i) = |p_i| = |\mathbf{X}_i|$,
and for a rule $r \in \prog$, let $a = \max_{p(\mathbf{X}) \in \ruleHead \cup \ruleBody} |\mathbf{X}|$ be the maximum arity.
$\text{IsVar}(x)$ evaluates to true iff the term $x$ is a variable.
We furthermore define $\var(r){\,\eqdef\,} \{x \mid x\in \vecv{X}, p(\vecv X)\in H_r \cup \ruleBody, \text{IsVar}(x)\}$.
For non-ground rules we define \ruleHead, \rulePosBody, \ruleNegBody, and \ruleBody as in the ground case,
as we do with the attributes \emph{disjunctive}, \emph{normal}, \emph{constraint}, \emph{stratified}, \emph{tight}, and \emph{HCF}.
The size of a rule is $|r| = |\ruleHead \cup \ruleBody|$
and of a program $|\prog| = \sum_{r \in \prog} |r|$.
Grounding is the instantiation of the variables by their domain.
Let $\mathcal{F} = \{p(\vecv{D}) \mid p(\vecv{D}) \in \prog, \forall d \in \vecv{D}: \neg \text{IsVar}(d)\}$ be the facts and
$\dom(\prog) = \{d \mid p(\vecv{D}) \in \mathcal{F}, d \in \vecv{D}\}$ be the domain.
Let $x$ be a variable, then $\dom(x) = \dom(\prog)$.
\emph{Naive grounding} $\mathcal{G}_N(\prog)$ instantiates for each rule all variables by all possible domain values, which results in a grounding size 
in $\mathcal{O}\left( |\prog| \cdot |\dom(\prog)|^{\max_{r \in \prog} |\var(r)|} \right)$.
For non-ground programs the herbrand base $\text{HB}(\prog)$ is defined as $\text{HB}(\prog) = \{p(\vecv{D}) \mid r \in \mathcal{G}_N(\prog), p(\vecv{D}) \in \ruleHead \cup \ruleBody\}$.
The semantics of a non-ground program $\prog$ is defined over its ground version $\mathcal{G}_N(\prog)$ and carries over from the ground case.

The non-ground dependency graph~$\depGraph_{\prog}$ of the non-ground program $\prog$ carries over from the ground case and is defined over the predicates.
$\scc(\prog)$ refers to
the set of \emph{strongly-connected components (vertices)} 
of~$\mathcal{D}_\prog$.
A reduced graph $\depGraph_{R}(G)$ of a graph $G = (V,E)$
is $\depGraph_{R}(G) = (V_r,E_r)$, where $V_r = \scc(G)$ 
and
$E_r = \{(s_1,s_2) \mid s_1,s_2 \in \scc(G), s_1 \not = s_2, \exists v_1 \in s_1 \exists v_2 \in s_2: (v_1,v_2) \in E\}$.
Any reduced graph is a directed acyclic graph (DAG).
Let $p$ be a predicate and $L_{\prog}$ be a topological order of the reduced dependency graph $\depGraph_R(\depGraph) = (V_r, E_r)$
and let $\scc_{\prog}(p)$ be the function $\scc_{\prog}(p):V \rightarrow V_r$ 
that returns the corresponding SCC of $p$, i.e., $\scc_{\prog}(p) = s$ s.t. $s \in \scc(\prog)$ and $p \in S$. 
Let $s = \scc_{\prog}(p)$ and $S_{\prec p}(0) = \{s\}$.
We iteratively extend $S_{\prec p}$ to a fixed point by $S_{\prec p}(t+1) = \{s | s \in \scc(\prog), \exists s' \in S_{\prec p}(t): (s,s') \in E_r\} \cup S_{\prec p}(t)$ for $t > 0$.
A fixed point is reached when $S_{\prec p}(t+1) = S_{\prec p}(t)$,
which we denote as $S_{\prec p} = S_{\prec p}(t)$.
As $\depGraph_{R}(G)$ is a DAG, such a fixed point always exists~\cite{tarski_lattice-theoretical_1955,knaster_theoreme_1928}.
A predicate $p$ is stratified iff $\forall s \in S_{\prec p}$,
there is no cycle with at least one negative edge in $s$.
Further, let \textit{IsStratified(r)} be true iff $r$ contains (only) stratified body predicates $p \in \ruleBody$.
Let \textit{IsTight(r)} be true iff $\forall h \in \ruleHead: \forall p \in \rulePosBody: \scc_{\prog}(h) \not = \scc_{\prog}(p)$ - so $r$ occurs in a tight part.
The variable graph $\varGraph{r} = (V,E)$ for a rule $r \in \prog$ is defined as the undirected graph where
$V = \var(r)$ and $E = \{(x_i,x_j) \mid x_i,x_j \in \var(r), \exists p(\vecv{X}) \in \ruleHead \cup \ruleBody: \{x_i,x_j\} \subseteq \vecv{X}\}$.
A tree decomposition (TD) $\mathcal{T} = (T, \chi)$ is defined over an undirected graph $G = (V,E)$ where $T$ is a tree and $\chi$ a labeling function 
$\chi: T \rightarrow V$. $\chi(t) \subseteq V$ is called a bag.
A tree decomposition must fulfill: (i) $\forall v \in V \exists t \in T: v \in \chi(t)$,
(ii) $\forall (u,v) \in E \exists t \in T: \{u,v\} \subseteq \chi(t)$, and 
(iii) every occurrence of $v \in V$ must form a connected subtree in T w.r.t. $\chi$, so
$\forall t_1, t_2, t_3 \in T$, s.t. whenever $t_2$ is on the path between $t_1$ and $t_3$,
it must hold $\chi(t_1) \cap \chi(t_3) \subseteq \chi(t_2)$.
The width of a tree decomposition is defined as the largest cardinality of a bag minus one, so $\max_{t \in T} |\chi(t)| - 1$.
The treewidth (TW) is the minimal width among all TDs.
Further, let $\varphi_{r}$ denote the bag size of a minimal tree decomposition of the variable graph of $r$.

\textbf{Bottom-up/Semi-naive grounding}.
Grounders \gringo and \idlv use (bottom-up) semi-naive database instantiation techniques to ground a program $\prog$~\cite{gebser_grounding_2015,calimeri_i-dlv_2017}.
In the following, we sketch the intuition.
Let $L_{\prog}$ be a topological order of $G_R(\depGraph_{\prog})$, and
let $D$ be the \emph{candidate set}, where $D \subseteq \text{HB}(\prog)$; initially $D = \mathcal{F}$.
Intuitively, the candidate set $D$ keeps track of all possibly derivable literals
and is iteratively expanded by moving along the topological order $L_{\prog}$.
For each $v \in L_{\prog}$ 
rules are instantiated according to the candidate set $D$ by a fixed-point algorithm.
If a tuple is in $D$ it is possibly true,
conversely, if a tuple is not in $D$, it is surely false.
If an SCC contains a cycle, semi-naive techniques are used 
to prevent unnecessary derivations~\cite{gebser_grounding_2015,calimeri_i-dlv_2017}.
The grounding size is exponential in the maximum number of variables $\mathcal{O}\left( \sum_{r \in \prog} |\dom(\prog)|^{|\var(r)|} \right)$ in the worst-case.
We use the terms \emph{state-of-the-art (SOTA)}, traditional, bottom-up, or semi-naive grounding interchangeably.

\textbf{Bottom-up grounding solves stratified programs}.
Bottom-up grounding is typically implemented in a way that enables full
evaluation of stratified programs.
Technically, this is implemented by partitioning the candidate set $D$
into a surely derived set $D_T$ and a potentially derived set $D_{pot}$.
Conversely, for any $a \in \text{HB}(\prog)$, but $a \not \in D_{pot} \cup D_T$, we know that we can never derive $a$.
This split leads to a series of improvements related to instantiating rules,
among them is the full evaluation of stratified programs.
However, these improvements have 
no effect on the grounding size of non-stratified programs in the worst case, 
thereby remaining 
exponential in the variable number.

\textbf{Structure-aware rewritings}.
Utilizing the rule structure to rewrite non-ground rules is performed by \lpopt~\cite{morak_preprocessing_2012,bichler_lpopt_2016}.
It computes a minimum size tree decomposition,
which is then used to introduce fresh rules with a preferably smaller grounding size.
In more detail,
for every rule $r \in \prog$ \lpopt first creates the
variable graph \varGraph{r}.
After computing a minimum-size tree decomposition,
it introduces fresh predicates and fresh rules for every bag of the tree decomposition.
The arity of the fresh predicates corresponds to the respective bag size,
as does the number of variables per rule.
Let $\textit{TW}(\varGraph{r})$ be the maximum treewidth of all rules $r \in \prog$,
then $\varphi_r = \textit{TW}(\varGraph{r}) + 1$ is its bag size.
It was shown that \lpopt produces a rewriting that is exponential in $\varphi_r$,
where $\varphi_r \leq \max_{r \in \prog} |var(r)|$:
$\mathcal{O}(|\prog| \cdot |\dom(\prog)|^{\varphi_r})$.
Internally, \idlv uses the concepts of \lpopt to reduce the grounding size~\cite{calimeri_optimizing_2018}.

\textbf{Body-decoupled Grounding}.
Body-decoupled grounding (BDG)~\cite{besin_body-decoupled_2022} produces grounding sizes that are exponential only in the maximum arity.
Conceptually, BDG decouples each rule into its literals which are subsequently grounded.
As each literal has at most arity-many variables, its grounding size can be at most exponential in its arity.
Semantics is ensured in three ways:
(i) For a rule $r$, all possible values of its head literals are guessed, and
(ii) satisfiability and (iii) foundedness are ensured by explicitly encoding them.
Interoperability with other techniques is ensured by hybrid grounding~\cite{beiser_bypassing_2024}.

Let $\prog$ be an HCF program and $\prog_{\HybridGrounding} \cup \prog_{\SOTAGrounding}$ be a partition thereof.
Then, let $\HybridGrounding$ be the \emph{Hybrid Grounding} procedure that is executed on $(\prog_{\HybridGrounding}, \prog_{\SOTAGrounding})$, 
where $\prog_{\HybridGrounding}$ is grounded by BDG and $\prog_{\SOTAGrounding}$ is grounded by bottom-up grounding.
Let $a$ be the maximum arity ($a = \max_{r\in\prog} \max_{p(\mathbf{X}) \in \ruleHead \cup \ruleBody} |\mathbf{X}|$)
and let $c$ be a constant defined as:
where $c =a$ for $r$ being a constraint,
$c=2\cdot a$ for $r$ occurring in a tight HCF program, and
$c=3 \cdot a$ for $r$ occurring in an HCF program.
Then, hybrid grounding for $\HybridGrounding(\prog, \emptyset)$ has a grounding size\footnote{For brevity we sometimes shorten $\mathcal{O}\left(|\prog| \cdot |\dom(\prog)|^{x} \right)$
with $\approx |\dom(\prog)|^{x}$ for an arbitrary $x \in \mathbb{N}$. }
of $\approx |\dom(\prog)|^{c}$.
The coefficients $c$ stem from the nature of the checks we have to perform.
For constraints, it is sufficient to check satisfiability,
while for normal programs we additionally need to check foundedness, 
which increases the grounding size to $c=2\cdot a$.
For HCF programs, cyclic derivations must be prevented.
This is handled with level-mappings,
where the transitivity check increases the grounding size to $c=3 \cdot a$.

\vspace{-0.3cm}
\section{Automated Splitting Heuristics}
\label{chpt:automated-splitting:sec:automated-splitting-heuristics}

We designed an automated splitting heuristics that decides when it is beneficial to use BDG. This approach is given in Algorithm~\ref{alg:data-heur}.
Intuitively, the decision is based on fixed structural measures, like the number of variables and treewidth,
as well as data-driven grounding-size estimation.
Let $\prog$ be an HCF program, and $r \in \prog$,
then
let $\hat{T}_{\mathcal{H}}(r)$ be the estimated grounding size of BDG, and
let $\hat{T}_{\mathcal{\bowtie}}(r)$ be the estimated SOTA grounding size.
The algorithm takes as input a rule $r$ and the set \MarkerSet.
Set \MarkerSet stores whether a rule $r$ is grounded by BDG or SOTA if $(r,\text{BDG}) \in \text{\MarkerSet}$ or $(r, \text{SOTA}) \in \text{\MarkerSet}$ respectively.
This is then used to pass $\prog_{\HybridGrounding} = \{r \mid r \in \prog, (r,\text{BDG}) \in \text{\MarkerSet} \}$
and $\prog_{\SOTAGrounding} = \{r \mid r \in \prog, (r,\text{SOTA}) \in \text{\MarkerSet}\}$
to 
$\HybridGrounding$. 

First, in Lines~(\ref{chpt:aut-split:alg:data-heur-line-2})--(\ref{chpt:aut-split:alg:data-heur-line-3}),
the algorithm performs a stratification check, where rules are SOTA-grounded whenever rules occur in stratified parts.
Subsequently, the rule structure is checked, and a structural rewriting is performed
in Lines~(\ref{chpt:aut-split:alg:data-heur-line-4})--(\ref{chpt:aut-split:alg:data-heur-line-7}), if beneficial.
Finally, in Lines~(\ref{chpt:aut-split:alg:data-heur-line-9})--(\ref{chpt:aut-split:alg:data-heur-line-16}).
BDG is evaluated and marked whenever it is structurally and data-estimation-wise beneficial.

\begin{algorithm}[t]
    \KwData{Rule $r$, Set  \text{\MarkerSet} of marked rules}
    \uIf{IsStratified($r$)} { \label{chpt:aut-split:alg:data-heur-line-2}
        \MarkerSet $\leftarrow \text{\MarkerSet} \cup (r, \text{SOTA})$ \; \label{chpt:aut-split:alg:data-heur-line-3}
    }
    \uElseIf{$\varphi_r < |\var(r)| \wedge \hat{T}_{\bowtie}(\text{\lpopt}(r)) < \hat{T}_{\bowtie}(r)$ } { \label{chpt:aut-split:alg:data-heur-line-4}
        $R_{l} \leftarrow \text{\lpopt}(r)$ \; \label{chpt:aut-split:alg:data-heur-line-1}
        \For{$r_l$ in $R_{l}$} { \label{chpt:aut-split:alg:data-heur-line-5}
            $\textit{Heur}_{\text{struct}}(r_l, \text{\MarkerSet})$ \; \label{chpt:aut-split:alg:data-heur-line-6}
        }\label{chpt:aut-split:alg:data-heur-line-7}
    }\label{chpt:aut-split:alg:data-heur-line-8}
    \uElseIf{$a < \varphi_r \wedge $ IsConstraint($r$) $ \wedge \hat{T}_{\mathcal{H}}(r) < \hat{T}_{\bowtie}(r)$} { \label{chpt:aut-split:alg:data-heur-line-9}
        \MarkerSet $\leftarrow \text{\MarkerSet} \cup (r, \text{BDG})$ \; \label{chpt:aut-split:alg:data-heur-line-10}
    } 
    \uElseIf{$2 \cdot a < \varphi_r \wedge $ IsTight($r$) $ \wedge \hat{T}_{\mathcal{H}}(r) < \hat{T}_{\bowtie}(r)$ } { \label{chpt:aut-split:alg:data-heur-line-11}
        \MarkerSet $\leftarrow \text{\MarkerSet} \cup (r, \text{BDG})$ \; \label{chpt:aut-split:alg:data-heur-line-12}
    }
    \uElseIf{$3 \cdot a < \varphi_r \wedge \hat{T}_{\mathcal{H}}(r) < \hat{T}_{\bowtie}(r)$  } { \label{chpt:aut-split:alg:data-heur-line-13}
        \MarkerSet $\leftarrow \text{\MarkerSet} \cup (r, \text{BDG})$ \; \label{chpt:aut-split:alg:data-heur-line-14}
    }
    \Else{ \label{chpt:aut-split:alg:data-heur-line-15}
        \MarkerSet $\leftarrow \text{\MarkerSet} \cup (r, \text{SOTA})$ \; \label{chpt:aut-split:alg:data-heur-line-16}
    } \label{chpt:aut-split:alg:data-heur-line-17}
    \caption{$\textit{Heur}(r,\text{\MarkerSet})$ for Computing Data-Structural Heuristics}
    \label{alg:data-heur}
\end{algorithm}

\begin{example}
    \label{chpt:automatic-splitting:fig:estimation-size}
    We show the details and underlying intuitions of the heuristics along the lines of the example shown below.
    A simple instance graph is given by means of atoms over the edge predicate $e/2$.
    We guess subgraphs $f/2$, $g/2$, and $h/2$,
    where we forbid three or more connected segments in subgraph $f/2$,
    cliques of size $\geq 3$ in subgraph $g/2$,
    and aim at inferring all vertices of a clique of size $\geq 3$ in subgraph $h/2$.    
    Let $r_1$, $r_2$, $r_3$ be the rule in Line~(2), (3), (4), respectively.
      \begin{lstlisting}
{f(X,Y)} :- e(X,Y). {g(X,Y)} :- e(X,Y). {h(X,Y)} :- e(X,Y).
$\;\;$:- f(X1,X2), f(X2,X3), f(X3,X4).
$\;\;$:- g(X1,X2), g(X1,X3), g(X2,X3).
i(X1) :- h(X1,X2), h(X1,X3), h(X2,X3).
    \end{lstlisting}
\end{example}
    
Previous results indicate that BDG should be used for \textit{dense rules} on \textit{dense instances}~\cite{besin_body-decoupled_2022,beiser_bypassing_2024}.
However, the terms \textit{dense rule} and \textit{dense instance} were loosely defined and
the usage of BDG was guided by intuition.
Our algorithm makes these terms precise and \emph{transitions from intuition to computation}.

\smallskip
\noindent \textbf{Variable-based Denseness}. Next, we motivate how we consider variable-based denseness.
\begin{example}
Observe how $r_1$ has four and $r_2$, and $r_3$ have three variables.
Standard bottom-up grounding is exponential in these variables in the worst case.
Without considering contributions of data and structural based rewritings for now,
bottom-up's grounding size for rule $r_1$ is
$\approx |\dom(\prog)|^{4}$,
while it is $\approx |\dom(\prog)|^{3}$ for $r_2$, and
$\approx |\dom(\prog)|^{3}$ for $r_3$.
In contrast, BDG's grounding size is only dependent on the maximum arity and the type of the rule.
The maximum arity of all $r_1,r_2$, and $r_3$ is $2$.
As both $r_1$ and $r_2$ are constraints, their grounding size is in $\approx |\dom(\prog)|^{2}$,
while as $r_3$ is a tight HCF rule
its grounding size is $\approx |\dom(\prog)|^{3}$.
The differences between BDG and SOTA are striking: A reduction from $\approx |\dom(\prog)|^4$ to $\approx |\dom(\prog)|^2$ and
from $\approx |\dom(\prog)|^3$ to $\approx |\dom(\prog)|^2$ 
for $r_1$ and $r_2$, respectively (no difference for $r_3$).
\end{example}
\vspace{-0.2cm}

We cover \textit{variable-based denseness}
based on the rule type and a comparison between the number $|\var(r)|$ of the variables and the maximum arity $a$.
Henceforth, whenever the maximum arity adjusted for rule type is strictly smaller than the number of the variables, BDG is used.
Let the maximum head arity be $a_h = \max_{p(\mathbf{X}) \in \ruleHead} |\mathbf{X}|$
and the maximum body arity be $a_b = \max_{p(\mathbf{X}) \in \ruleBody} |\mathbf{X}|$.
For constraints, using BDG is beneficial whenever $a < |\var(r)|$,
for tight HCF rules if $a_h + a_b \leq 2 \cdot a < |\var(r)|$,
and for HCF rules if $3 \cdot a < |\var(r)|$.

When the projected grounding sizes match asymptotically,
precedence is given to the bottom-up procedure: First, due to the effects of data (discussed below) and second,
due to BDG's nature of pushing effort from grounding to solving.
Since bottom-up grounding solves stratified programs with a grounding size in $\approx |\dom(\prog)|^a$, grounding stratified parts with BDG is not beneficial.

\medskip
\noindent\textbf{Incorporating Rule Structure}. To grasp the importance of structure, recall our running example.

\vspace{-0.2cm}
\begin{example}
We depict the variable graphs of $r_1$, $r_2$, and $r_3$ in Figure~\ref{fig:var-graph-rules},
which have treewidths of $1$, $2$, and $2$ respectively.
A minimal tree decomposition of the variable graph of $r_1$ has a bag size of $\varphi_{r_1} = 2$.
Take for example $\mathcal{T} = (T,\chi)$, where $T = \left(\{t_1,t_2,t_3\},\{\{t_1,t_2\},\{t_2,t_3\}\}\right)$ and $\chi(t_1) = \{X1,X2\}$,
$\chi(t_2) = \{X2,X3\}$, and $\chi(t_3) = \{X3,X4\}$.
Based on $\mathcal{T}$, we depict in the next listing a possible structural rewriting.
Observe the grounding size of $\approx |\dom(\prog)|^2$.
\begin{lstlisting}
tmp1(X3) :- f(X3,X4). tmp2(X2) :- f(X2,X3), tmp1(X3). :- f(X1,X2), tmp2(X2).
\end{lstlisting}
In contrast to this, a minimal tree decomposition of $r_2$ or $r_3$ has a bag size of $\varphi_r = 3$,
such as $\mathcal{T} = (T,\chi)$, where $T = \left(\{t_1\},\emptyset\right)$ and $\chi(t_1) = \{X1,X2,X3\}$.
Using structural rewritings for $r_2$ or $r_3$ has no effect.
%
%
Therefore, the grounding sizes of BDG and \lpopt match for $r_1$ (both are $\approx |\dom(\prog)|^2$),
while BDG achieves a reduction from $\approx |\dom(\prog)|^3$ to $\approx |\dom(\prog)|^2$ for $r_2$.
For $r_3$, both have a grounding size of $\approx |\dom(\prog)|^3$.
Whenever grounding sizes of BDG and \lpopt match, we give preference to \lpopt,
as for BDG there are guesses\footnote{Guesses are due to Equations~(2), (4), and (9) of Figure~1 in hybrid grounding~\cite{beiser_bypassing_2024}.} during solving.
\end{example}
\vspace{-0.2cm}

The observations above are incorporated in the heuristics by computing the treewidth of its variable graph and using \lpopt whenever the bag size $\phi_r$
of a minimal tree decomposition is strictly smaller than the number $|\var(r)|$ of variables ($\phi_r < |\var(r)|$).
See Lines~(\ref{chpt:aut-split:alg:data-heur-line-4})--(\ref{chpt:aut-split:alg:data-heur-line-7}).
Subsequently, a decision between BDG and bottom-up grounding is made
based on the bag size of a minimal decomposition compared to the maximum arity of~$r$ ($a < \phi_r$),
and the rule-type (constraint, tight, non-tight).
Thereby, we transition from variable-based denseness to structure-aware denseness,
which we incorporate into our algorithm in Lines~(\ref{chpt:aut-split:alg:data-heur-line-9}),~(\ref{chpt:aut-split:alg:data-heur-line-11}), and~(\ref{chpt:aut-split:alg:data-heur-line-13}).
%
%
\begin{figure}[t]
    %
    \centering
    \includegraphics[width=9.5cm]{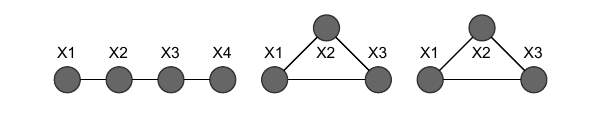}
    \vspace{-2em}
    \caption{Variable Graphs of $r_1$ (left), $r_2$ (center), and $r_3$ (right) for Example~\ref{chpt:automatic-splitting:fig:estimation-size}.}
    \label{fig:var-graph-rules}
\end{figure}

\smallskip
\noindent \textbf{Incorporating Data-Awareness}.
The incorporation of data into our heuristics is vital.
In its absence, BDG may be used when it is unwise to use it.
Indeed, BDG is a domain-based grounding procedure, whose grounding size depends entirely on the domain of the program.
On the other hand, bottom-up grounding is partially data-aware, as rule bodies perform joins between variables.

\vspace{-0.2cm}
\begin{example}
To visualize this, consider $r_2$ and a graph that is a path with $100$ vertices.
%
While BDG's grounding size of $r_2$ is $\approx |100|^2$,
bottom-up's grounding size is $0$.
\end{example}
\vspace{-0.2cm}

To incorporate data into heuristics, observe that rule instantiations are similar to joins in a database system,
where joins are done in the positive body~\cite{leone_improving_2001}.
Interestingly, join size estimation procedures are common in the literature~\cite{garcia-molina_database_2008}.
We estimate the SOTA grounding size according to the join-selectivity criterion~\cite{leone_improving_2001}\footnote{A variant of the join-selectivity criterion is used in \idlv~\cite{calimeri_optimizing_2018}.}.
%

Let $r \in \prog$.
We compute the join estimation $\hat{T}_{\bowtie}(r)$ in an iterative way,
by considering one literal $p_{i} \in \rulePosBody$ at a time.
We start with the first positive body literal $p_{l+1}$ and end with the
last positive body literal $p_{m}$, as $\rulePosBody = \{p_{l+1}, \ldots, p_m\}$.
Further, we denote the computation of all positive predicates up to and including $p_{i}$ as $A_i$.
Let $\hat{T}(p_{i+1})$ be the estimated number of tuples of $p_{i+1}$,
and $\hat{T}(A_i)$ be the estimated join size up to and including predicate $p_{i}$.
Let $\dom(X,r)$ be the domain of variable $X$ for the rule $r$,
$\dom(X,p_i)$ be the domain of variable $X$ for literal $p_i$,
and let $p_X$ be $p_X = \{p(\mathbf{X}) \mid p(\mathbf{X}) \in \rulePosBody, X \in \mathbf{X}\}$, where $X \in \var(r)$ is a variable.
We compute a variable's domain size as $\dom(X,r) = \bigcup_{p_i(\mathbf{X}) \in p_X} \dom(X,p_i)$.
Equations~(\ref{eq:estimate-join-base})--(\ref{eq:estimate-join2}) show our join size estimation for SOTA-grounding for a rule $r$,
where $\hat{T}_{\bowtie}(r)$ refers to the estimation for a rule $r$.

~\\[-2.5em]
\begin{align}
    \label{eq:estimate-join-base} \hat{T}(A_{l+1}) &= \hat{T}(p_{l+1})\\
    \label{eq:estimate-join} \hat{T}(A_{i+1}) &= \hat{T}(A_i \bowtie p_{i+1}) = \frac{\hat{T}(A_i) \cdot \hat{T}(p_{i+1})}{\Pi_{X \in \var(A_i) \cap \var(p_{i+1})} |\dom(X,r)| } \\
    \label{eq:estimate-join2} \hat{T}_{\bowtie}(r) &= \hat{T}(A_m) = \hat{T}(A_{m-1} \bowtie p_m)
\end{align}

%
%

Precise grounding size estimations are possible for hybrid grounding.
We show in Equations~(\ref{red-est:head})--(\ref{red-est:found-3}) the grounding size estimations for non-ground normal (HCF) programs.
Each \emph{equation} estimates the size of the respective hybrid grounding \emph{rules}\footnote{To avoid confusion, we distinguish in this paragraph between \emph{equation}, the grounding size estimation, and \emph{rule}, the equation of the hybrid grounding reduction that is being estimated as introduced in~\citeN{beiser_bypassing_2024}.}, as introduced in~\citeN{beiser_bypassing_2024}.
Consider for example Equation~(\ref{red-est:sat-3}),
which estimates the size of Rules~(5)--(7) of the hybrid grounding reduction as introduced in~\citeN{beiser_bypassing_2024}.
It intuitively captures for a rule $r \in \prog$ whether a literal $p(\mathbf{X}) \in \ruleHead \cup \ruleBody$
for an arbitrary instantiation $p(\mathbf{D}) \in \text{HB}(\prog)$ contributes to $r$ being satisfied.
We estimate this as $\hat{T}_{\mathcal{H}}^{S3}(r)$ in Equation~(\ref{red-est:sat-3}).
We continue with a brief description of the other equations and their corresponding rules in the hybrid grounding reduction.
Equation~(\ref{red-est:head}) is the estimation of the head-guess size,
for the respective Rule~(2).
Equations~(\ref{red-est:sat-1})--(\ref{red-est:sat-3}) estimate the size of the satisfiability encoding,
where Equations~(\ref{red-est:sat-1}) and~(\ref{red-est:sat-2}) estimate the impact of variable guessing, saturation, and the constant parts,
which relate to the Rules~(4) and~(8) in hybrid grounding.
We already described Equation~(\ref{red-est:sat-3}) above.
Equations~(\ref{red-est:found-1})--(\ref{red-est:found-3}) estimate the size of the foundedness part.
Equation~(\ref{red-est:found-1}) estimates the size of the constraint that prevents unfounded answersets, which relates to Rule~(12).
Equation~(\ref{red-est:found-2}) estimates the size of the variable instantiations, which relates to  Rule~(9).
Finally, Equation~(\ref{red-est:found-3}) is concerned with the estimation when a rule is suitable for justifying an atom, which relates to Rules~(10)--(11).
\vspace{-0.6cm}

\begin{align}
    \label{red-est:head}& \hat{T}_{\mathcal{H}}^{G}(r) = 2 \cdot \left( \Sigma_{h(\mathbf{X}) \in H_r} \Pi_{X \in \mathbf{X}} |\dom(X)| \right)\\
    \label{red-est:sat-1}& \hat{T}_{\mathcal{H}}^{S1}(r) = 2 \cdot \Sigma_{X \in \var(r)} |\dom(X)|\\ 
    \label{red-est:sat-2}& \hat{T}_{\mathcal{H}}^{S2}(r) = 2\\
    \label{red-est:sat-3}& \hat{T}_{\mathcal{H}}^{S3}(r) = \Sigma_{p(\mathbf{X}) \in \ruleHead \cup \ruleBody} \Pi_{X \in \mathbf{X}} |\dom(X)| \\
    \label{red-est:found-1}& \hat{T}_{\mathcal{H}}^{F1}(r) = \Sigma_{h(\mathbf{X}) \in H_r} \Pi_{X \in \mathbf{X}} |\dom(X)|\\
    \label{red-est:found-2}& \hat{T}_{\mathcal{H}}^{F2}(r) = \Sigma_{h(\mathbf{X}) \in H_r} \left( \Sigma_{Y \in \var(r) \setminus \mathbf{X}} \left(|\dom(Y)| \cdot \Pi_{X \in \mathbf{X}} |\dom(X)| \right) \right)\\
    \label{red-est:found-3}& \hat{T}_{\mathcal{H}}^{F3}(r) = \Sigma_{h(\mathbf{X}) \in H_r} \left( \Sigma_{p(\mathbf{Y}) \in \ruleHead \cup \ruleBody \setminus \{h(\mathbf{X})\}} \left(\Pi_{Y \in \mathbf{Y}} |\dom(Y)| \cdot \Pi_{X \in \mathbf{X}} |\dom(X)| \right) \right)
\end{align}
\vspace{-0.6cm}

We are left with Equation~(\ref{red-est:tot}), which computes $\hat{T}_{\mathcal{H}}(r)$,
the hybrid grounding size estimation for a rule $r$.
Equation~(\ref{red-est:tot}) sums up Equations~(\ref{red-est:head})--(\ref{red-est:found-3}).

\vspace{-0.6cm}
\begin{align}
    \label{red-est:tot}& \hat{T}_{\mathcal{H}}(r) = \hat{T}_{\mathcal{H}}^{G}(r) + \hat{T}_{\mathcal{H}}^{S1}(r) + \hat{T}_{\mathcal{H}}^{S2}(r) + \hat{T}_{\mathcal{H}}^{S3}(r) + \hat{T}_{\mathcal{H}}^{F1}(r) + \hat{T}_{\mathcal{H}}^{F2}(r) + \hat{T}_{\mathcal{H}}^{F3}(r)
\end{align}

\begin{figure}[t]
    \centering
    \begin{subfigure}[t]{0.49\textwidth}
        \centering
        \includegraphics[width=6.5cm]{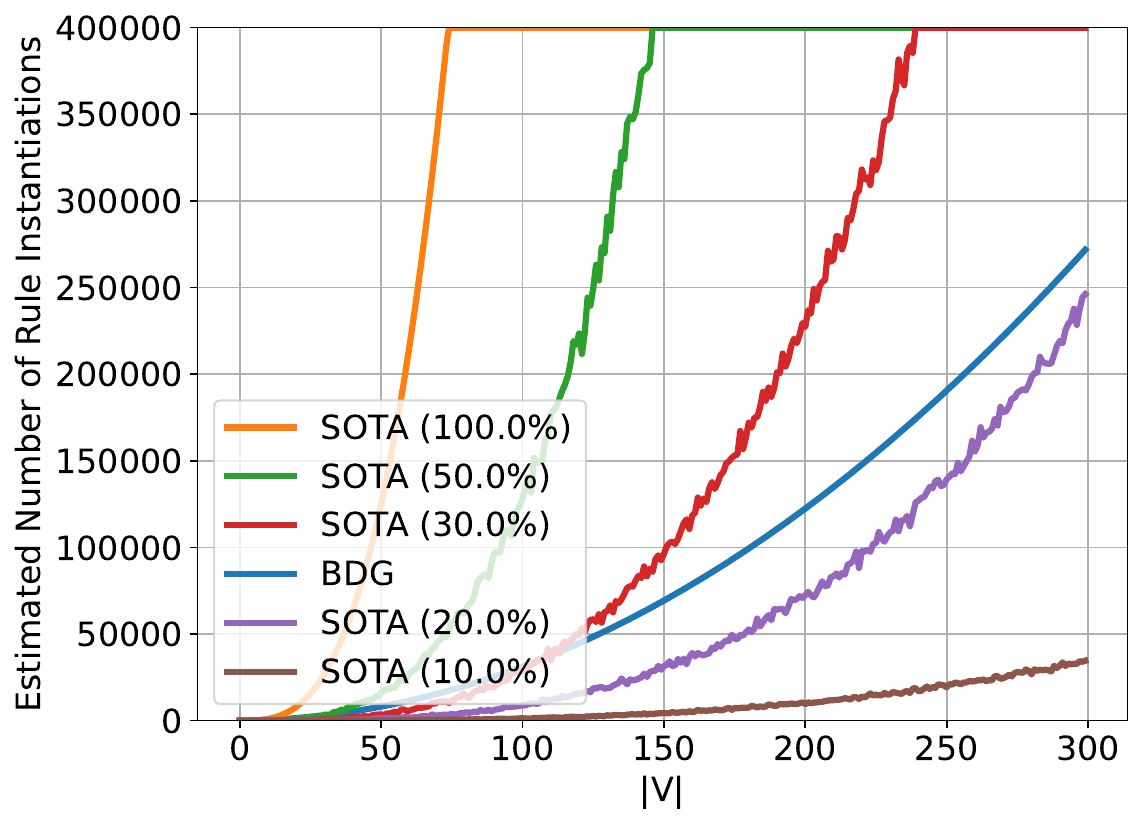}\vspace{-.5em}
    \end{subfigure}
    \begin{subfigure}[t]{0.49\textwidth}
    \centering
        \includegraphics[width=6.5cm]{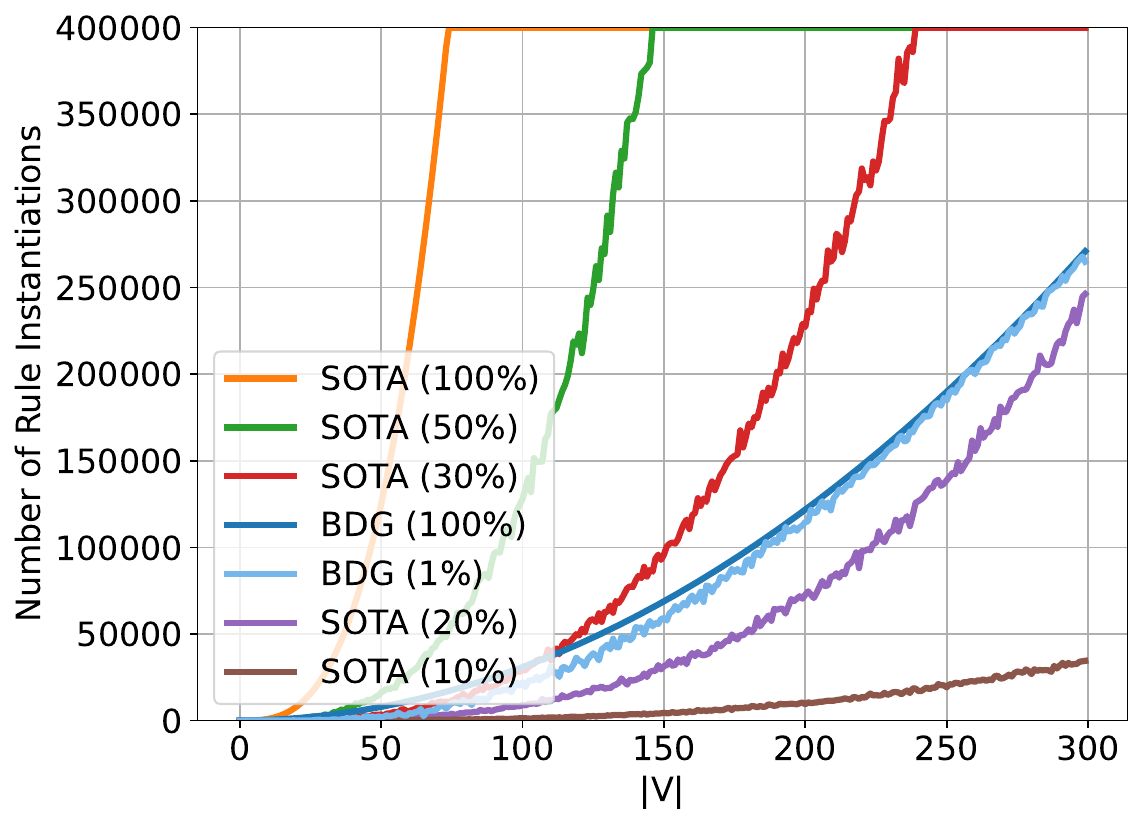}\vspace{-.5em}
    \end{subfigure}
\label{chpt:automatic_splitting:fig:example_estimation}
\vspace{-.75em}
    \caption{
            Plot comparing the estimated (left) and actual (right) number of ground rules of $r_2$ of Example~\ref{chpt:automatic-splitting:fig:estimation-size}.
            Comparison between SOTA and BDG.
            x-axis: Number of vertices; y-axis: Number of rules.
            Comparing different graph densities,
            shown as SOTA($x$) and BDG($x$) for density $x$.
            }
        \label{fig:automatic_splitting:fig:example_estimation}
\end{figure}

\begin{example}
In Figure~\ref{fig:automatic_splitting:fig:example_estimation} we show the estimated and actual number of instantiated rules for bottom-up grounding and BDG,
for $r_2$.
The behavior is analyzed on different graph densities (number of edges divided by edges of complete graph in percent)
and graph sizes (1 to 300 vertices).
The number of tuples $T(p_i)$ can be adequately
estimated for our example, so $\hat{T}(p_i) \approx T(p_i)$.
While for bottom-up grounding the estimated number of ground rules varies with  density,
it remains constant for BDG.
BDG's number of instantiated rules between a complete ($100\%$) and a sparse ($1\%$) graph 
remains relatively similar.
For bottom-up grounding, the number of instantiated rules varies.
\end{example}

\medskip 
\noindent Overall we obtain the following result on the grounding size by automated hybrid grounding.

\begin{theorem}
    \label{chpt:automatic-split:thm:grounding-size-data-struct-heur}
    Let \prog be a non-ground HCF program and
    $k$ be the maximum treewidth of any rule in $\prog$.
    Then, the grounding size of $\prog$,
    grounded with the markings \MarkerSet,
    $\prog_{\HybridGrounding} = \{r \mid r \in \prog, (r,\text{BDG}) \in \text{\MarkerSet} \}$ and
    $\prog_{\SOTAGrounding} = \{r \mid r \in \prog, (r,\text{SOTA}) \in \text{\MarkerSet}\}$,
    produced by Algorithm~\ref{alg:data-heur} and grounded
    by  $\HybridGrounding(\prog_{\HybridGrounding}, \prog_{\SOTAGrounding})$,
    is in
    $\mathcal{O}\left((|\prog| \cdot k) \cdot |\dom(\prog)|^{3 \cdot a} \right)$.
\end{theorem}
\begin{proof}[Proof (idea)]
Intuitively, structural parts of the algorithm bound the grounding size to
$\mathcal{O}\left((|\prog| \cdot k) \cdot |\dom(\prog)|^{3 \cdot a} \right)$.
We are left to prove that this still holds when incorporating data-awareness,
which holds on dense instances.
The proof is detailed in the appendix.
\end{proof}

\section{Prototype Implementation \ngrThree}
\label{sec:novel-prototype}

Our prototype \ngrThree\footnote{Prototype available under \url{https://github.com/alexl4123/newground}.} is a full-fledged grounder that combines bottom-up with body-decoupled grounding.
It incorporates BDG into the bottom-up procedure, where we decide according to the data-structural heuristics 
(Algorithm~\ref{alg:data-heur})
whether to use BDG or not.
Furthermore, the algorithm does not pre-impose on the user which SOTA grounder to use,
and therefore, offers integration with \gringo and \idlv.
In this section, we discuss implementation choices, highlight implementation challenges,
and present the structure of the prototype.

%
%
We performed a full-scale redevelopment of the earlier versions of \ngrThree (\ngr and \nag),
where on a high level, semi-naive grounding is interleaved with body-decoupled grounding.
We further extended its input language to the ASP-Core-2~\cite{calimeri_asp-core-2_2020} input language standard\footnote{
    Currently not all ASP-Core-2 constructs are supported with BDG rewritings. Checks ensure that only 
    supported constructs are considered to be grounded by BDG, while non-groundable ones are grounded by SOTA-techniques.
    }
and improved the grounding performance of \ngr.
For the semi-naive grounding parts we use either \gringo, or \idlv,
whereas, for the BDG part we use a completely redesigned BDG-instantiator.
To improve performance even further, we combine Python with Cython and C code.

\smallskip
\textbf{Architectural Overview}.
\begin{figure}[t]
    \centering
    \includegraphics[width=8cm]{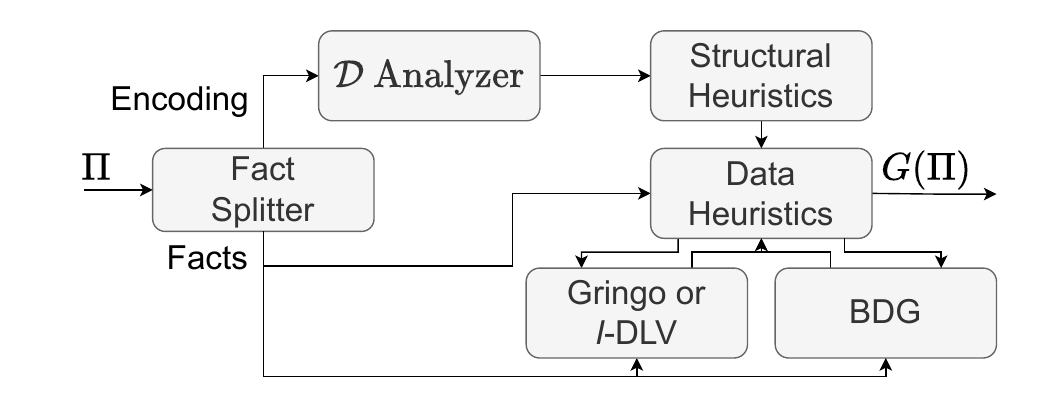}
\vspace{-1.25em}
    \caption{
        Schematics of the software architecture of the \ngrThree prototype.
    }\label{chpt:automatic_splitting:fig:prototype-schematics}
\end{figure}
The general architecture of the prototype consists of $4$ parts, where we show a schematics in Figure~\ref{chpt:automatic_splitting:fig:prototype-schematics}.
Given a program $\prog$, the \textit{fact splitter and analyzer} (Fact Splitter) written in Cython, separates facts from the encoding. 
It further computes the number of facts, and fact-domain.
This enables an efficient computation of the \textit{positive dependency graph and analysis thereof} ($\mathcal{D}$ Analyzer).
Based on these results the \textit{structural heuristics} decides which rules are eligible for grounding with BDG.
If no rules are structurally eligible for grounding with BDG then the program is grounded by either \gringo or \idlv.
Otherwise, the bottom-up procedure is \emph{emulated} and for each strongly connected component in the positive dependency graph, where
at least one rule is structurally eligible for grounding with BDG,
the data heuristics decides whether to ground the rule with BDG or with a SOTA-approach.

In the development of the prototype we encountered two major challenges:
(i) integration and communication with \gringo and \idlv, and
(ii) suitable domain inference for grounding size estimations of Algorithm~\ref{alg:data-heur}.
To address these, we split the data-structural heuristics into two parts in our implementation:
first, the structural heuristics decides, which parts are eligible for grounding with BDG
and only then the estimation of the size of the instantiation of the eligible rules is performed.
Further, we minimize the number of interactions with \gringo and \idlv,
as each call to a SOTA-grounder is expensive and should better be avoided.
Therefore, we do not infer the domain if the result of the structural heuristics states that BDG should not be used.
The emulation is necessary, as neither \gringo nor \idlv provides callback functions which let us implement our heuristics directly.
In the future a direct implementation of the heuristics in a SOTA grounder would render these calls unnecessary and would improve performance even further.

\section{Experiments}
\label{sec:automatic-splitting:experiments}

In the following, we demonstrate the practical usefulness of our automated hybrid grounding approach.
We benchmark solving-heavy and grounding-heavy instances, aiming at SOTA-like results on solving-heavy benchmarks,
and beating SOTA results on grounding-heavy benchmarks.

\smallskip
\textbf{Benchmark System}.
We compared \gringo (Version 5.7.1), \idlv (1.1.6), \ProASP (Git branch \textit{master}, short commit hash \textit{2b42af8}), \AlphaG (Version 0.7.0), and our hybrid grounding system \ngrThree.
We benchmarked \ngrThree with both \gringo, and \idlv.
Further, we investigated the impact of using our system in combination with \lpopt (Version 2.2).
We chose \clingo (Version 5.7.1) with \clasp (3.3.10) for solving.
However, in principle, one could also use \dlv with \wasp,
or use heuristics to determine the solver of choice~\cite{calimeri_efficiently_2020}.
For \ngrThree we use Python version 3.12.1.
Our system has \textit{225 GB} of RAM,
and an \textit{AMD Opteron 6272} CPU,
with 16 cores, powered by \textit{Debian 10} OS with kernel 4.19.0-16-amd64.

\smallskip
\textbf{Benchmark Setup}.
For all experiments and systems, we measure \textit{total time},
which includes grounding and solving time for ground-and-solve systems,
or execution time for \AlphaG and \ProASP.
Further, we measure RAM usage for all systems and experiments.
For the ground-and-solve systems we measured grounding performance (grounding time, grounding size, and RAM usage)
in a separate run.
Every experiment has a timeout of \textit{1800s} and a RAM (and grounding-size) limit of \textit{10 GB}.
For integrated grounders and solvers (\AlphaG and \ProASP) this RAM limit applies to their execution.
For ground-and-solve systems this applies to grounding and solving.

We consider instances as a \textit{TIMEOUT} whenever they take longer than \textit{1800s},
and a \textit{MEMOUT} when their RAM usage exceeds \textit{10 GB}.
We set seeds for clingo (\textit{11904657}),
and
for \lpopt (\textit{11904657}).
Further, for all generated graph instances for the grounding-heavy experiments
we generated random seeds that we saved inside the random instance
as a predicate.

%
\begin{figure}[t]
        \hspace{-.3em}\begin{subfigure}[b]{0.48\textwidth}
            \centering
            \includegraphics[width=6.8cm]{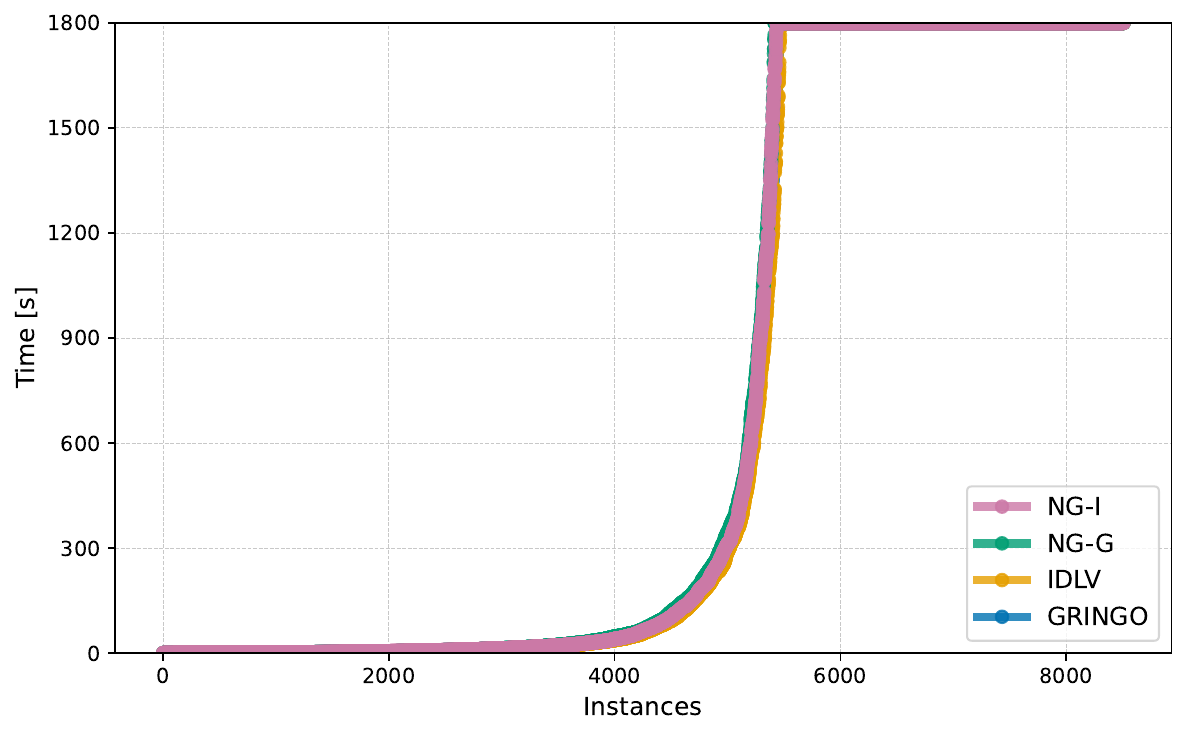}\vspace{-.5em}
            \caption{
                Solving-heavy: Ground \& Solve Time~[s].
            }
            \label{chpt:auto:subfig:solving-heavy-cactus-tt}
        \end{subfigure}
        \hspace{0.1cm}
        \begin{subfigure}[b]{0.48\textwidth}
            \centering
            \includegraphics[width=6.8cm]{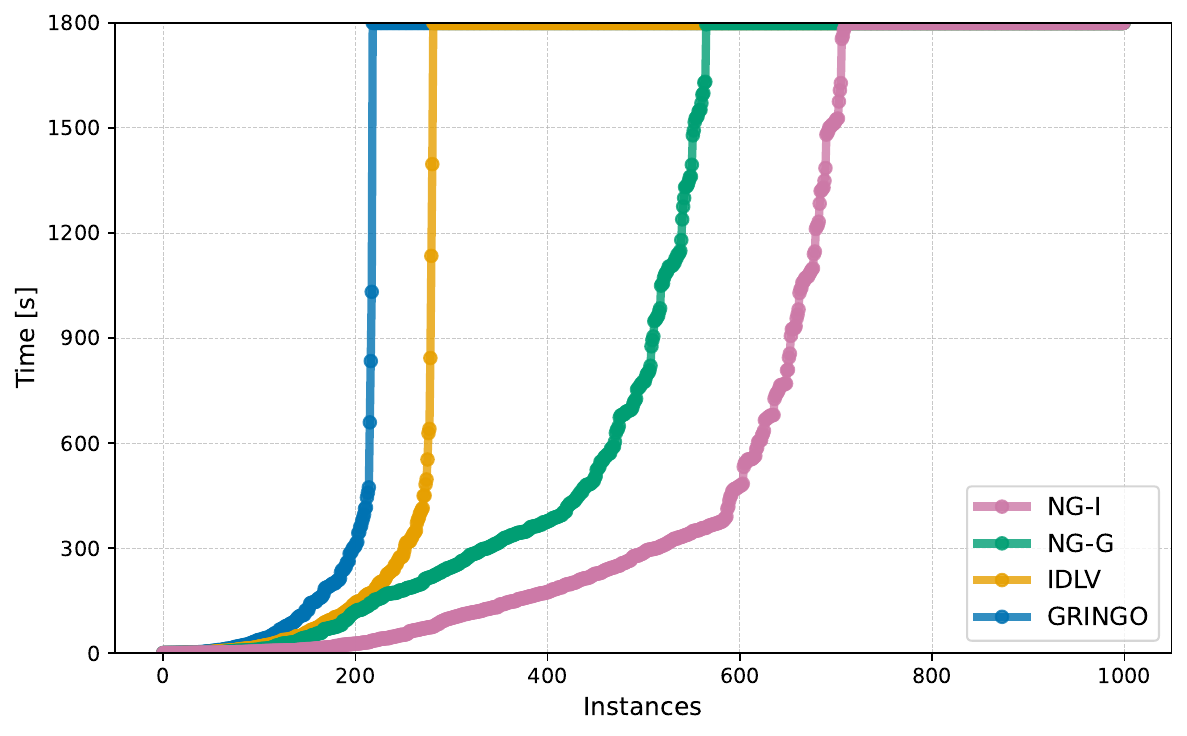}\vspace{-.5em}
            \caption{
                Grounding-heavy: Ground \& Solve~Time~[s].
            }
            \label{chpt:auto:subfig:grounding-heavy-cactus-tt}
        \end{subfigure}

        \hspace{-.3em}\begin{subfigure}[b]{0.48\textwidth}
            \centering
            \includegraphics[width=6.8cm]{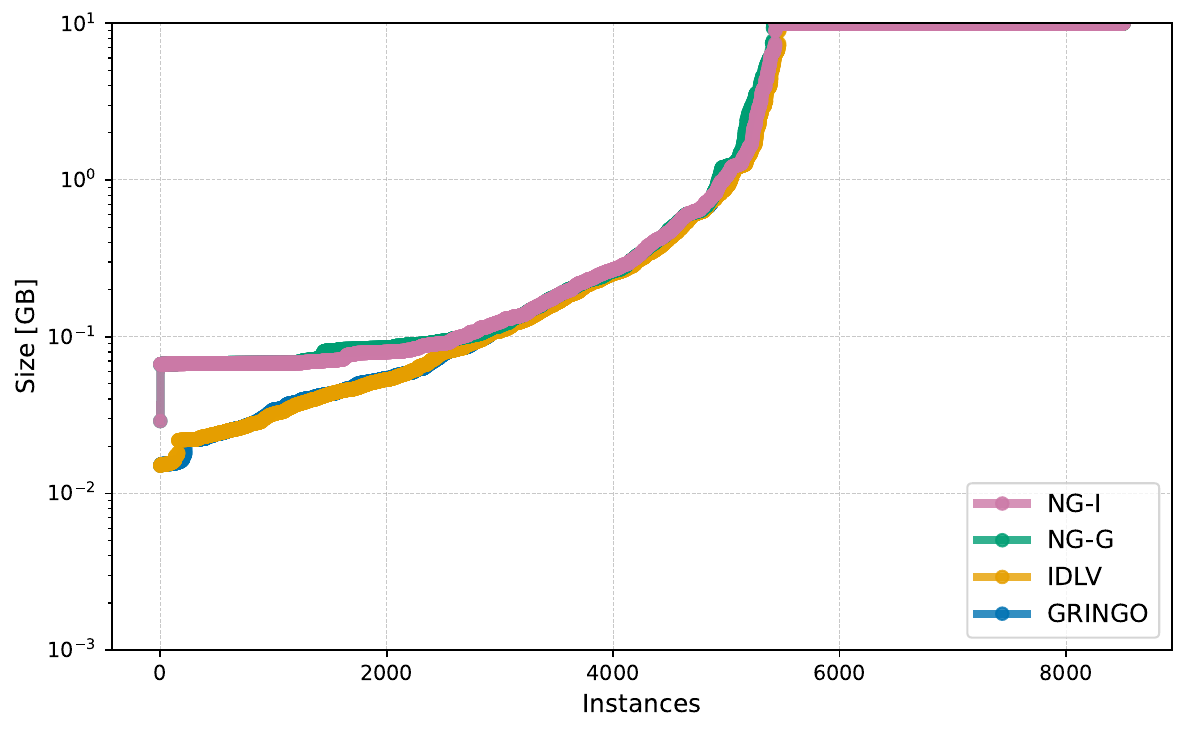}\vspace{-.5em}
            \caption{
                Solving-heavy: Max RAM Usage [GB].
            }
            \label{chpt:auto:subfig:solving-heavy-cactus-ru}
        \end{subfigure}
        \hspace{0.1cm}
        \begin{subfigure}[b]{0.48\textwidth}
            \centering
            \includegraphics[width=6.8cm]{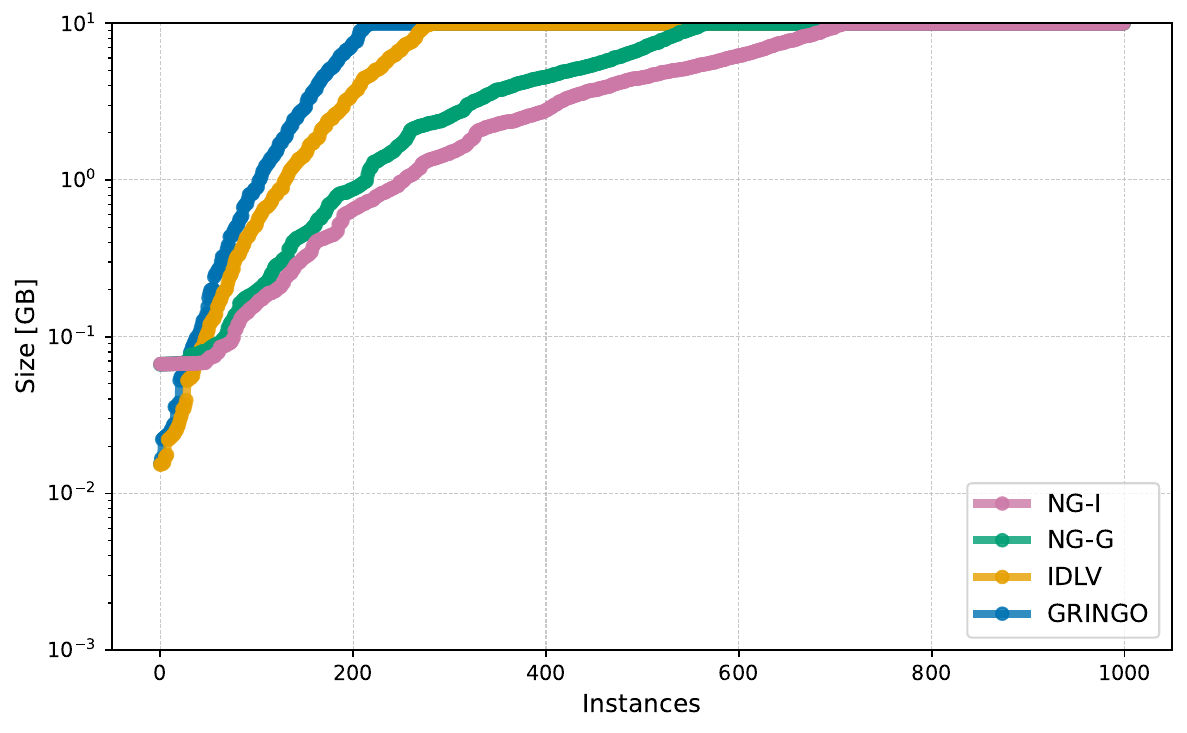}\vspace{-.5em}
            \caption{
                Grounding-heavy: Max RAM Usage [GB].
            }
            \label{chpt:auto:subfig:grounding-heavy-cactus-ru}
        \end{subfigure}
\vspace{-.75em}
    \caption{
    Solving-heavy
        (Figures~\ref{chpt:auto:subfig:solving-heavy-cactus-tt}, and~\ref{chpt:auto:subfig:solving-heavy-cactus-ru})
        and grounding-heavy (Figures~\ref{chpt:auto:subfig:grounding-heavy-cactus-tt}, and~\ref{chpt:auto:subfig:grounding-heavy-cactus-ru}) experiments.
        x-axis: Instances; y-axis: Time [s] or Size [GB].
        Measured \idlv, \gringo,
        \ngrThree with \gringo (\NgrThreeG), and
       \ngrThree with \idlv (\NgrThreeI).
        Timeout: 1800s; Memout: 10GB.
    } 
    \label{chpt:auto:fig:solving-heavy-cactus}
\end{figure}

\vspace{-1em}
\subsection{Experiment Scenarios and Instances}

We distinguish between \textit{solving}- and \textit{grounding}-heavy benchmarks.
For the solving-heavy benchmarks we compare \idlv, \gringo,
\ngrThree with \gringo (\NgrThreeG),
\ngrThree with \idlv (\NgrThreeI),
\AlphaG, and
\ProASP (ground-all).
%
For the grounding-heavy benchmarks we compare grounders \idlv, \gringo,
\ngrThree with \gringo,
\ngrThree with \idlv,
\AlphaG,
\ProASP (ground-all), 
and
\ProASP with compiling constraints (\ProASP\texttt{-CS}).

\smallskip
\textbf{Solving-Heavy Benchmarks}.
The solving-heavy benchmarks are taken from the 2014 ASP-Competition~\cite{calimeri_design_2016},
as they provide a large instance set with readily available efficient encodings.
The 2014 ASP-Competition has 25 competition scenarios,
where each (with the exception of \textit{Strategic-Companies}) has two encodings,
resulting in 49 competition scenarios.
Each scenario has a different number of instances.
We benchmarked all instances over all scenarios.
Further, we preprocessed the encodings s.t. no predicates occur,
which have the same predicate name, but differing arity.

We show the encoding of
problem \textit{MaximalCliqueProblem} (2014 encoding)\footnote{
    The whole competition suite can be found at: \url{https://www.mat.unical.it/aspcomp2014/FrontPage}
} as an example:
\begin{lstlisting}
clique(X) :- node(X), not nonClique(X).
nonClique(X) :- node(X), not clique(X).
$\;\;$:- clique(X), clique(Y), X < Y, not edge(X,Y), not edge(Y,X).
:$\sim$ nonClique(X). [1,X]
\end{lstlisting}

Intuitively the encoding guesses nodes that are part of the maximal clique (Lines~1,2).
If there is a missing edge between a pair of nodes, then it is not a clique (Line~3).
We 
minimize the number of non-clique nodes (Line~4).

\smallskip\textbf{Grounding-Heavy Benchmarks}.
We take grounding-heavy benchmarks from the BDG experiments~\cite{besin_body-decoupled_2022}
and from the hybrid grounding experiments~\cite{beiser_bypassing_2024}.
These scenarios take a graph as an input,
where we generate random graphs ranging from instance size $100$ to $2000$ with a step-size of $100$ for the BDG scenarios~\cite{besin_body-decoupled_2022}
and random graphs ranging from instance size $20$ to $400$ with a step-size of $20$ for the hybrid grounding scenarios~\cite{beiser_bypassing_2024}.
For both, we use graph density levels ranging from $20\%$ to $100\%$.

Further, we adapt the benchmarks from~\cite{besin_body-decoupled_2022} by adding two variations of the \emph{3-Clique} benchmark.
The variations resemble different difficulties for BDG and SOTA grounders.
The first listing (3-Clique-not-equal) shows the original formulation 
from~\cite{besin_body-decoupled_2022},
and the second one (3-Clique) depicts the adaptation that makes it easier for SOTA grounders 
 by changing ``$\neq$'' to ``$<$''.
\begin{lstlisting}
{f(X,Y)} :- edge(X,Y).
$\;\;$:- f(A,B), f(A,C), f(B,C), A != B, B != C, A != C.
\end{lstlisting}
\begin{lstlisting}
{f(X,Y)} :- edge(X,Y).
$\;\;$:- f(A,B), f(A,C), f(B,C), A < B, B < C, A < C.
\end{lstlisting}

\begin{figure}[t]
        \hspace{-.5em}
        \begin{subfigure}[b]{0.51\textwidth}
            \centering
            \includegraphics[width=6.7cm]{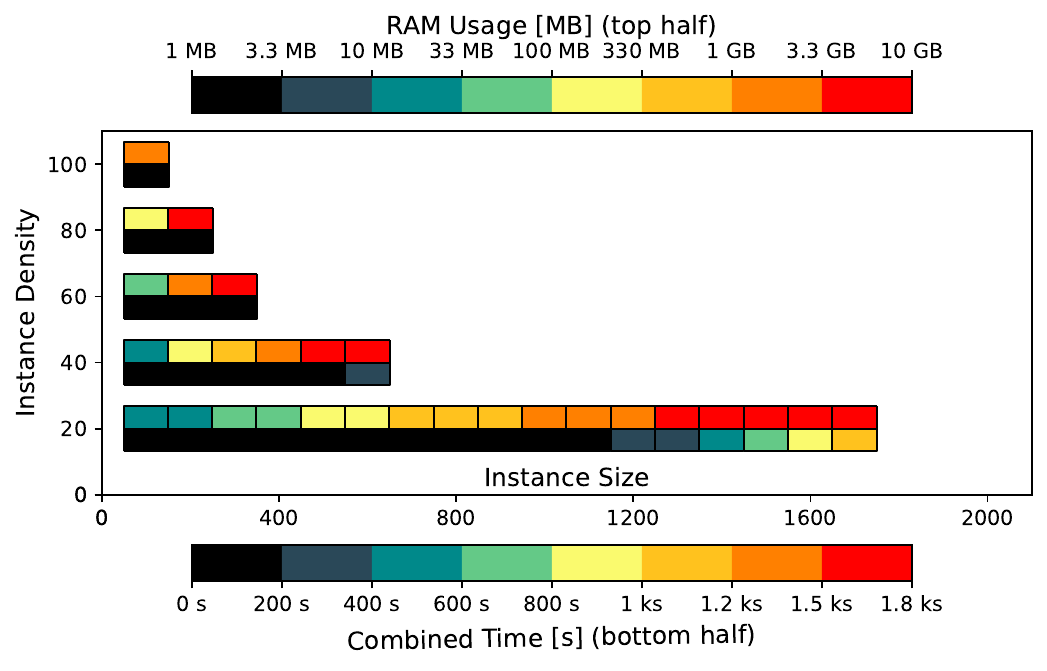}
        \end{subfigure}
         \begin{subfigure}[b]{0.48\textwidth}
            \centering
             \includegraphics[width=6.7cm]{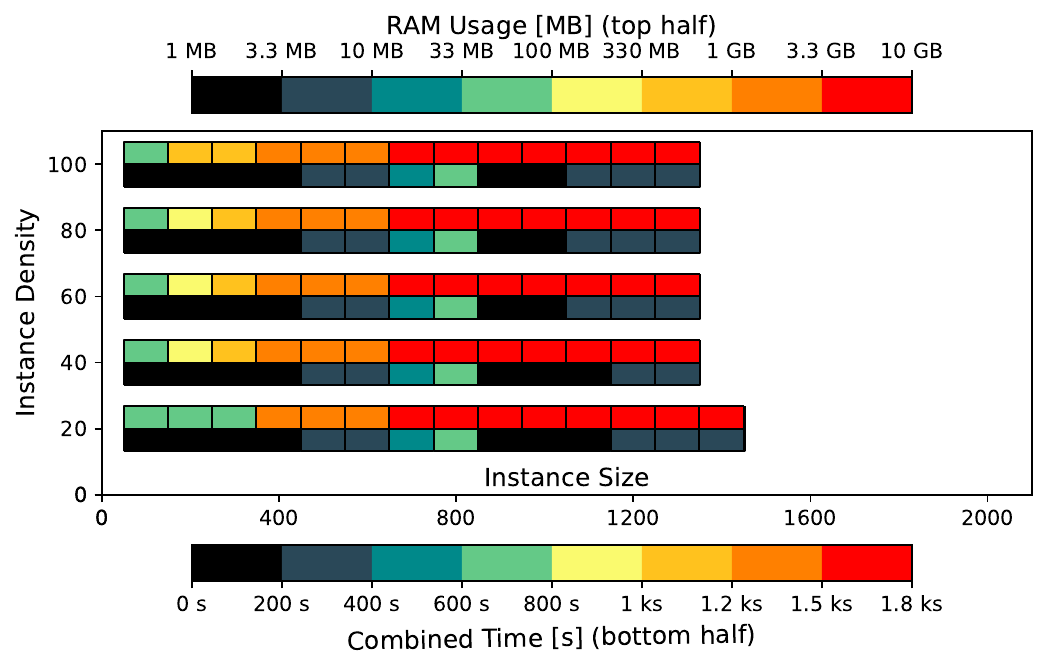}
        \end{subfigure}
\vspace{-1.1em}
        \caption{
        Solving profiles for grounding-heavy scenario \textit{4-Clique}
        for \gringo (left) and \ngrThree with \gringo (\NgrThreeG).
        One rectangle represents one grounded and solved instance.
            Timeout: 1800s; Memout: 10GB. Instance size on x-axis, instance density on y-axis.
            }
            \label{fig:solving-profile-33}
\end{figure}

The adapted\footnote{
\ProASP's syntax currently does not support choice rules,
so we adapted the subgraph encoding for \ProASP with a negative cycle encoding (\textit{f(X,Y) :- edge(X,Y), not nf(X,Y). nf(X,Y) :- edge(X,Y), not f(X,Y).}). This is also used for \AlphaG.
}
scenarios from~\cite{besin_body-decoupled_2022} are called as follows:
3-Clique, 3-Clique-not-equal, directed-Path, directed-Col, 4-Clique, NPRC.
The examples S3T4, S4T6, NPRC-AGG, and SM-AGG, are from~\cite{beiser_bypassing_2024}.

\subsection{Experimental Hypotheses}
\label{chpt:automated-hg:experimental-hypotheses}

\begin{itemize}
    \item[H1] The Data-Structural-Heuristics (Algorithm~\ref{alg:data-heur}) implemented in our prototype \ngrThree
    approaches other SOTA ground-and-solve system's performance on solving-heavy benchmarks.
    \item[H2] Data-Structural-Heuristics of \ngrThree
        yields an improvement in performance (solved instances) on grounding-heavy benchmarks,
        in contrast to other SOTA systems.
\end{itemize}

\subsection{Experimental Results and Discussion}
\label{chpt:automated-hg:experimental-results}

We show an overview of our results in Table~\ref{tbl:summary-table} and Figure~\ref{chpt:auto:fig:solving-heavy-cactus}; a detailed solving profile of the grounding-heavy scenario \emph{4-Clique}\ is given in Figure~\ref{fig:solving-profile-33}.
For details, see supplementary material.

\begin{table}[t]
\scalebox{.85}{
\begin{tabular}[t]{c|c|ccc|ccc|ccc}
\multirow{1}{*}{Instance Summary} & \multirow{1}{*}{\#I} & \multicolumn{9}{c}{Total \#{Solved}} \\ 
\toprule
& & \multicolumn{9}{c}{\textbf{Grounding-Heavy Scenarios}} \\ \cmidrule(lr){3-11}
 &  & \multicolumn{3}{c}{\gringo} & \multicolumn{3}{c}{\idlv} & \multicolumn{3}{c}{\NgrThreeG} \\ \cmidrule(lr){3-11} 
& & \#S & M & T & \#S & M & T & \#S & M & T  \\ 
\cmidrule(lr){3-11}
All & 1000 & 218 & 169 & 613 & 281 & 198 & 521 & 566 & 336 & 98 \\ 
ProASP & 500 & 149 & 97 & 254 & 158 & 102 & 240 & 280 & 210 & 10 \\ \midrule 
& & \multicolumn{3}{c}{\NgrThreeI} & \multicolumn{3}{c}{\AlphaG} & \multicolumn{3}{c}{\ProASP\texttt{-CS}}\\ \cmidrule(lr){3-11}
& & \#S & M & T & \#S & M & T & \#S & M & T  \\ \cmidrule(lr){3-11}
All & 1000 & \textbf{710} & 247 & 43 & - & - & - & - & - & - \\ 
ProASP & 500 & 288 & 198 & 14 & 147 & 177 & 176 & \textbf{389} & 81 & 30 \\ \midrule \midrule
& & \multicolumn{9}{c}{\textbf{Solving-Heavy Scenarios}} \\ \cmidrule(lr){3-11}
 &  & \multicolumn{3}{c}{\gringo} & \multicolumn{3}{c}{\idlv} & \multicolumn{3}{c}{\NgrThreeG} \\ \cmidrule(lr){3-11} 
& & \#S & M & T & \#S & M & T & \#S & M & T \\ \cmidrule(lr){3-11} 
All & 8509 & 5449 & 650 & 2410 & \textbf{5469} & 697 & 2343 & 5418 & 524 & 2567 \\ 
Alpha & 1640 & 1255 & 30 & 355 & \textbf{1280} & 0 & 360 & 1251 & 24 & 365 \\ 
ProASP & 320 & 308 & 0 & 12 & 308 & 0 & 12 & 307 & 0 & 13 \\ \midrule 
& & \multicolumn{3}{c}{\NgrThreeI} & \multicolumn{3}{c}{\AlphaG} & \multicolumn{3}{c}{\ProASP}\\ \cmidrule(lr){3-11}
& & \#S & M & T & \#S & M & T & \#S & M & T  \\ \cmidrule(lr){3-11}
All & 8509 & 5434 & 599 & 2476 & - & - & - & - & - & - \\
Alpha & 1640 & 1272 & 0 & 368 & 183 & 290 & 1167 & - & - & - \\
ProASP & 320 & 306 & 0 & 14 & 3 & 53 & 264 & \textbf{311} & 0 & 9 \\ \midrule
\end{tabular}}
\vspace{-1.1em}
\caption{
Experimental results showing all scenarios, those executable by Alpha, and those executable by ProASP, with differing number of instances (\#I).
We depict solved instances (\#S), Memouts (M), and Timeouts (T) for \gringo, \idlv, \NgrThreeG, \NgrThreeI, \AlphaG, and \ProASP.
}
\label{tbl:summary-table}
\end{table}

\textbf{Discussion of H1}.
To confirm H1, we focus our attention on the results of the solving-heavy experiments.
These are displayed in Figures~\ref{chpt:auto:subfig:solving-heavy-cactus-tt} and~\ref{chpt:auto:subfig:solving-heavy-cactus-ru} 
and in the lower half of the Table~\ref{tbl:summary-table}.
The figures show that that \ngrThree's performance is approximately the same as the other ground-and-solve approaches.
The detailed results of the table show that the overall number of solved instances for \gringo is $5449$,
for \idlv $5469$, for \NgrThreeG $5418$, and for \NgrThreeI $5434$.
The difference between \gringo and \NgrThreeG are $31$ instances, 
and for \idlv and \NgrThreeI are $35$ instances.
On in total $8509$ solving-heavy instances this resembles an approximate relative
difference of
$0.36\%$ for \gringo vs. \NgrThreeG and
$0.41\%$ for \idlv vs. \NgrThreeI.
The detailed results show that for \gringo vs. \NgrThreeG there are cases where \gringo beats \NgrThreeG and cases where \NgrThreeG beats \gringo.
The same holds for \idlv vs. \NgrThreeI.
As the differences of solved instances between \ngrThree and the respective state-of-the-art grounders are minor, we confirm H1. 

\smallskip
\textbf{Discussion of H2}.
We compare the results for the grounding-heavy scenarios of Figures~\ref{chpt:auto:subfig:grounding-heavy-cactus-tt} and~\ref{chpt:auto:subfig:grounding-heavy-cactus-ru}, and the upper half of Table~\ref{tbl:summary-table}.
While \gringo solves $218$, and \idlv $281$, \ngrThree solves $566$ in the \NgrThreeG and $710$ in the \NgrThreeI configuration,
from a total of $1000$ instances.
This is a difference of $34.8\%$ and $42.9\%$, respectively.
Also observe the milder increase in RAM usage in Figure~\ref{chpt:auto:subfig:grounding-heavy-cactus-ru}
and the ability to ground denser instances (Figure~\ref{fig:solving-profile-33}).
As \ngrThree's ability to automatically determine when to use BDG leads to an approximate doubling
in the number of solved grounding-heavy instances, we can confirm H2.

\subsubsection*{Summary of Results}

For both solving-heavy and grounding-heavy benchmarks \NgrThreeG and \NgrThreeI outperformed \AlphaG significantly.
\ProASP has a comparable performance on solving-heavy benchmarks.
On grounding-heavy benchmarks, \ProASP shows promising results,
however only when we use \ProASP in the compile constraints mode.
In the ground-all mode its behavior is similar to \gringo or \idlv.
This confirms the results of previous studies about the performance of \ProASP~\cite{dodaro_blending_2024}.
Although the results of \ProASP are very promising, it is only usable for a small fragment of the scenarios.
%
%
%
%
%
%

\section{Conclusion}
\label{sec:discussion_conclusion}

The advancement of alternative grounding procedures is an important step towards solving the grounding bottleneck.
Previous results for the newly introduced body-decoupled grounding (BDG) method~\cite{besin_body-decoupled_2022} showed improvements on grounding-heavy tasks.
Hybrid grounding~\cite{beiser_bypassing_2024} 
enables manual partitioning of a program into a part grounded by standard grounders and a part grounded by BDG.
However, due to the challenging predictability of BDG's solving performance,
it remained unclear when the usage of BDG is useful.

In this paper, we state a data-structural heuristics,
which decides when it is beneficial to use BDG.
Our heuristics decision is based on
the structure of a rule and the data of the instance.
For each rule a minimum tree decomposition of the rule's variable graph is computed
and compared to the maximum arity of the rule.
Whenever the bag size of the minimum tree decomposition is smaller, the rule is grounded with bottom-up grounding.
Otherwise the grounding size of the rule is estimated for bottom-up grounding by methods from databases,
which is compared to an estimate of BDG's grounding size.
Whichever is smaller is chosen for grounding.
Our prototype \ngrThree implements this heuristics by emulating a bottom-up procedure.
The results of our experiments show that we approach bottom-up grounders number of solved instances for solving-heavy scenarios,
while we approximately double the number of solved instances for grounding-heavy scenarios.
We think that this is an important step towards integrating BDG into state-of-the-art grounders.
However, there is still future work to be explored for BDG.
We argue that near-term research should include improvements of BDG for high-arity programs, as well as for syntactic extensions, highly cyclic rules, large HCF rules, and disjunctive programs.

\subsection*{Acknowledgments}
This research was funded in part by the Austrian Science Fund (FWF), grants 10.55776/COE12 and J 4656.
This research was supported by Frequentis.

\bibliographystyle{acmtrans}


\bibliography{paper}

\end{document}